\pdfoutput=1
\documentclass{scrartcl}

\usepackage[english]{babel}

\newenvironment{linenomath}{}{} 

\usepackage{amsmath}
\usepackage{amsthm}
\usepackage{amssymb}
\usepackage{stmaryrd}

\usepackage{url}
\usepackage{bookmark}
\usepackage{natbib}
\usepackage{enumitem}
\usepackage{xifthen}

\newcommand\mlabel[1]{\label{#1}}

\theoremstyle{definition}
\newtheorem{definition}{Definition}[section]
\newtheorem{lemma}[definition]{Lemma}
\newtheorem{theorem}[definition]{Theorem}
\newtheorem{corollary}[definition]{Corollary}

\newtheorem{proposition}[definition]{Proposition}

\newcommand{\var}[1]{\ensuremath{\operatorname{var}(#1)}}
\newcommand{\lit}[1]{\ensuremath{\operatorname{lit}(#1)}}

\newcommand{\range}[1]{\ensuremath{\operatorname{range}(#1)}}

\newcommand\vek[1]{\mathbf{#1}}
\newcommand\assign[1]{\mathbf{#1}}

\newcommand{\xdrenc}[2]{\ensuremath{\operatorname{DR^+}(#1, #2)}}

\newcommand{\litVar}[1]{\ensuremath{\llbracket #1\rrbracket}}

\newcommand{\meta}[1]{\ensuremath{\operatorname{meta}(#1)}}
\newcommand{\metaI}[2]{\ensuremath{\operatorname{meta}_{#1}(#2)}}

\renewcommand{\phi}{\varphi}
\newcommand{\metaVar}[2]{\ensuremath{\llbracket #2\rrbracket_{#1}}}

\newcommand{\leaf}[1]{\ensuremath{\operatorname{leaf}(#1)}}

\newcommand{\exonerep}[1]{\ensuremath{\operatorname{eo}(#1)}}

\newcommand{\encPC}{\ensuremath{\mathcal{P}}}

\makeatletter 
\newcommand\mynobreakpar{\par\nobreak\@afterheading} 
\makeatother

\usepackage{array}
\newcounter{nnfgrprowcntr}[table]
\renewcommand{\thennfgrprowcntr}{N\arabic{nnfgrprowcntr}}
\newcolumntype{N}{>{\refstepcounter{nnfgrprowcntr}{\thennfgrprowcntr}}c}

\makeatletter
\newcommand{\manuallabel}[2][]{\ifthenelse{\isempty{#1}}{#2\def\@currentlabel{#2}\label{#2}}{#1\def\@currentlabel{#1}\label{#2}}}
\makeatother

\usepackage{booktabs}

\usepackage{xcolor}
\definecolor{urlcolor}{HTML}{00009B}
\definecolor{citecolor}{HTML}{9B0000}

\hypersetup{
    colorlinks=true,       
    linkcolor=urlcolor,        
    citecolor=citecolor,         
    filecolor=magenta,     
    urlcolor=urlcolor,         
    linktocpage=true
}

\title{Backdoor Decomposable Monotone Circuits and CNF Encodings}
\author{Petr Ku{\v c}era\thanks{Department of Theoretical Computer Science
and Mathematical Logic,
Faculty of Mathematics and Physics,
Charles University, Czech Republic,
kucerap@ktiml.mff.cuni.cz}
   \and Petr Savick{\'y}\thanks{Institute of
      Computer Science of
   the Czech Academy of Sciences,
   Czech Republic,
   savicky@cs.cas.cz}
}
\date{\empty}

\begin{document}

\maketitle

\begin{abstract}
   We describe a compilation language of backdoor decomposable
   monotone circuits (BDMCs) which generalizes several concepts
   appearing in the literature, e.g.\ DNNFs and backdoor trees. A
   \(\mathcal{C}\)-BDMC sentence is a monotone circuit which satisfies
   decomposability property (such as in DNNF) in which the inputs (or
   leaves) are associated with CNF encodings from a given base class
   \(\mathcal{C}\). We consider the class of propagation complete (PC)
   encodings as a base class and we show that PC-BDMCs are
   polynomially equivalent to PC encodings. Additionally, we use this
   to determine the properties of PC-BDMCs and PC encodings with respect to
   the knowledge compilation map including the list of efficient
   operations on the languages.
\end{abstract}

\section{Introduction}

In knowledge compilation~\cite{DM02,M15}, we are concerned with
transforming a given propositional theory into a form which allows efficient
query answering and manipulation. The form of the output
representation is specified by a target compilation language. Lots of
target languages were described in knowledge compilation
map~\cite{DM02} which was later extended with other languages and
their disjunctive closures~\cite{FM08a}. We are in particular interested in
languages based on conjunctive normal forms (CNF) which are used to
encode various constraints into SAT\@. Since unit propagation is a
basic procedure used in DPLL based SAT solvers including CDCL solvers,
it has become a common practice to require that unit propagation
maintains at least some level of local consistency in the constraints
being encoded into a CNF formula.

Close connection between unit propagation in SAT solvers and maintaining
generalized arc consistency (GAC) was investigated for example
by~\citet{B07,BKNW09}. A stronger notion of propagation complete (PC)
encodings was introduced by~\citet{BM12} as a generalization of unit
refutation complete (URC) encodings~\cite{V94}. A formula \(\varphi\)
is propagation complete, if its consistency with a partial assignment
can be checked by unit propagation and in case the formula is
consistent with a partial assignment, unit propagation derives all
implied literals. When encoding a constraint into a CNF we usually
distinguish two kinds of variables, the main variables which
directly correspond to the variables of a constraint, and auxiliary
variables. A well known example is encoding of a circuit into a CNF
formula using Tseitin's encoding~\cite{T83} where the main variables
correspond to the inputs of the circuit and auxiliary variables
correspond to the gates. Let us note that PC
encodings treat all variables in the same way: the propagation
properties of a PC encoding with respect to the auxiliary variables is
the same as with respect to the main variables.
By way of contrast, if the encoding only maintains
GAC~\citep[see e.g.][]{B07} or, equivalently, domain consistency, the propagation
properties are required only for the main variables (GAC encoding in
this paper). A systematic study of encodings of multi-valued decision
diagrams (MDDs) with different
propagation strength is presented by~\cite{AGMES16} including a construction
of a polynomial size PC encoding for an arbitrary MDD.

Let us include a few remarks concerning
a weaker notion of consistency checker (CC encoding in this
paper) considered, for example, by~\cite{BKNW09,AGMES16}.
In the case of a CC encoding, unit propagation detects inconsistency
with a partial assignment of the main variables.
Following the framework of closures initiated by \citet{FM08a},
\citet{BJM12} studied disjunctive closures of different types
of CNF encodings and their placement into knowledge compilation map.
In particular, they consider CC encodings denoted as \(\exists\operatorname{\mathtt{URC-C}}\)
and one of their results is that the language of
disjunctions of CC encodings is polynomially equivalent to the
language of CC encodings. We generalize this by proving that
the language of disjunctions of PC encodings is polynomially
equivalent to the language of PC encodings.
Using a slightly simpler construction, we also prove that
the language of disjunctions of URC encodings is polynomially
equivalent to the language of URC encodings.

The above results are a consequence of a more general construction. We
introduce a new target compilation language parameterized with a class
$\mathcal{C}$ of encodings. For the special case when \(\mathcal{C}\)
is the class of PC encodings, we demonstrate a transformation of a
sentence in this language into a single PC encoding. 

A sentence in the proposed language
is a monotone circuit whose inputs called leaves are defined by
CNF encodings from a suitable base class \(\mathcal{C}\).
Additionally, the conjunctions in the circuit satisfy a
decomposability property with respect to the input variables similar
to the language of decomposable negation normal forms~\citep[DNNF, introduced
by][]{D99}. We call this structure backdoor decomposable
monotone circuit with respect to the base class \(\mathcal{C}\)
(\(\mathcal{C}\)-BDMC), because it is closely related to
\(\mathcal{C}\)-backdoor trees introduced by~\citet{SS08}.
By definition, the language of URC-BDMCs is a strict superset
of the language of disjunctions of URC encodings studied in \citet{BJM12}
as \(\operatorname{\mathtt{URC-C}}[\lor, \exists]\).

BDMC generalizes also other concepts appearing in the literature.
A DNNF can be
understood as a special case of $\mathcal{C}$-BDMC for any class
$\mathcal{C}$ containing the literals.
If we consider circuits with only one node, we obtain that PC-BDMC
sentences generalize PC formulas and URC-BDMC sentences generalize URC
formulas. In the rest of the paper, we
mostly consider PC-BDMCs, but the results can be transferred to
URC-BDMCs as well.
Note that monotone CNFs are propagation complete and thus they are
special cases of PC-BDMCs.
Combining it with the results of~\citet{BCMS14} or~\citet{BCMS16}
and the fact that DNNFs are also special cases of
PC-BDMCs, we obtain that the language of PC-BDMCs is strictly more
succinct than the language of DNNFs in the sense of the knowledge
compilation map~\cite[see][]{DM02}.

Backdoor trees introduced by~\citet{SS08}
were not primarily considered as a language for representing boolean
functions. Generalizing the idea of PC-backdoor trees towards
this goal, we directly
obtain a special case of PC-BDMCs in which the circuit part has form
of an out-arborescence in which we allow only
decision nodes and leaves are associated with PC encodings
without auxiliary variables. The difference from the original
backdoor trees is that the formulas in the leaves
are not necessarily restrictions of the same formula.
We show that PC-BDMCs are strictly more
succinct than the PC-backdoor trees generalized in this way.

The main result of our paper is that a PC-BDMC can be compiled into a
PC encoding of size polynomial with respect to the total size of the
input BDMC\@. As a consequence, we get that both PC-BDMCs and PC
encodings share the same algorithmic properties while being equally succinct. In
particular, we argue that the properties of PC-BDMCs and PC encodings
with respect to query answering and transformations described in the
knowledge compilation map~\cite{DM02} are the same as in the case of
DNNFs. At the same time, both PC-BDMCs and PC encodings are strictly
more succinct than DNNFs which in our opinion makes both these
language good target compilation languages.

A compilation of a smooth DNNF
into a URC or PC encoding of polynomial size was described
by~\citet{KS19} building on previous consistency checking and domain
consistency maintaining (or GAC) encodings described
by~\citet{AGMES16,GS12,BJKW08}.
We generalize these results to a more general structure, where the leaves
contain PC encodings instead of single literals.

\section{Definitions and Notation}
\label{sec:def} 

We assume the reader is familiar with the basics of propositional
logic, especially with the notion of entailment \(\models\) and the
notation related to formulas in conjunctive normal form (\emph{CNF
   formula}). We use \(\lit{\vek{x}}\) to denote the set of literals
(\(x\), \(\neg x\)) over the set of variables \(\vek{x}\). We treat a
clause as a set of literals and a CNF formula as a set of clauses. A
\emph{partial assignment} \(\alpha\) of values to variables in
\(\vek{z}\) is a subset of \(\lit{\vek{z}}\) that does not contain a
complementary pair of literals. A full assignment
\(\assign{a}:\vek{x}\to\{0, 1\}\) is a special type of partial
assignment and we use the representations by a function or by a set of literals
interchangeably.
We consider encodings of boolean functions defined as follows.

\begin{definition}[Encoding]
   \mlabel{def:cnf-enc} 
   Let \(f(\vek{x})\) be a boolean function on variables \(\vek{x}=(x_1, \dots,
      x_n)\). Let \(\varphi(\vek{x},\vek{y})\) be a CNF formula
   on \(n+m\) variables where \(\vek{y}=(y_1, \dots,
      y_m)\).
   We call \(\varphi\) a \emph{CNF encoding} of \(f\) if
   \begin{linenomath}
   \begin{equation}
      \label{eq:enc-def}
      f(\vek{x}) \equiv (\exists
      \vek{y})\, \varphi(\vek{x},
      \vek{y})\,.
   \end{equation}
   \end{linenomath}
   The variables
   in \(\vek{x}\) and \(\vek{y}\) are called \emph{main variables} and
   \emph{auxiliary variables}, respectively.
\end{definition}

We use \(\varphi\vdash_1 C\) to denote the fact that a clause \(C\)
can be derived by unit propagation interpreted as unit resolution
from a CNF formula \(\varphi\) (in particular, if \(C\in \varphi\),
then \(\varphi\vdash_1 C\)). The
notion of a propagation complete CNF formula was introduced
by~\citet{BM12} as a generalization of a unit refutation complete CNF
formula introduced by~\citet{V94}.

\begin{definition}[Propagation complete encoding]
   \mlabel{def:pc-enc} 
   Let \(\varphi(\vek{x}, \vek{y})\) be a CNF encoding of a boolean
   function defined on a set of variables \(\vek{x}\).
   We say that the encoding \(\varphi\) is \emph{propagation complete (PC)},
   if for every partial assignment
   \(\alpha\subseteq\lit{\vek{x}\cup\vek{y}}\) and for each
   \(l\in\lit{\vek{x}\cup\vek{y}}\), such that
   \begin{linenomath}
   \begin{equation}
      \label{eq:pc-enc-1}
      \varphi(\vek{x}, \vek{y})\land\alpha\models l
   \end{equation}
   \end{linenomath}
   we have
   \begin{linenomath}
   \begin{equation}
      \label{eq:pc-enc-2}
      \varphi(\vek{x}, \vek{y})\wedge\alpha\vdash_1 l
      \hspace{1em}\text{or}\hspace{1em}
      \varphi(\vek{x}, \vek{y})\wedge\alpha\vdash_1 \bot\,.
   \end{equation}
   \end{linenomath}
\end{definition}

\section{Backdoor Decomposable Monotone Circuits and the Main Result}
\label{sec:bdmc} 

In this section we introduce a language of backdoor decomposable
monotone circuits (BDMC) and state the main result of the paper.
BDMCs form a common generalization of
decomposable negation normal forms (DNNF) introduced by~\citet{D99} and
$\mathcal{C}$-backdoor trees introduced by~\citet{SS08} if the base
class $\mathcal{C}$ contains the literals as formulas. It consists of sentences formed by a
combination of a decomposable monotone circuit with CNF encodings from
a suitable class $\mathcal{C}$ at the leaves. More precisely, let
$\phi_i(\vek{x}_i, \vek{y}_i)$ for $i=1\ldots,\ell$ be encodings from
$\mathcal{C}$ with auxiliary variables \(\vek{y}_i\) whose main variables \(\vek{x}_i\)
are subsets of a set
of variables $\vek{x}$ and let us consider their combination by a
monotone circuit $D$ with $\ell$ inputs. This is a DAG with nodes \(V\),
root \(\rho\in V\), the set of edges \(E\), and the set of leaves
\(L\subseteq V\) of size \(\ell\). The inner nodes in $V$ are labeled
with \(\land\) or \(\lor\) and represent connectives or gates.
Each edge $(v,u)$ in \(D\) connects an inner node
$v$ labeled \(\land\) or \(\lor\) with one of its inputs $u$. The edge
is directed from $v$ to $u$, so the inputs of a node are its
successors (or child nodes).
We assume that there is a one to one correspondence between the leaves of
\(D\) and the formulas \(\varphi_i(\vek{x}_i, \vek{y}_i)\), \(i=1,
   \dots, \ell\) and we say that a leaf is labeled or associated with the
corresponding formula. For a given index \(i\in\{1, \dots, \ell\}\), the
leaf associated with \(\varphi_i\) is denoted as \(\leaf{i}\).
Given a literal \(l\in\lit{\vek{x}}\),
let us denote \(\range{l}\) the set of indices of formulas in the leaves which
contain variable \(\var{l}\), i.e.\
\(\range{l}=\{i\in\{1, \dots, \ell\}\mid \var{l} \in \vek{x}_i\}\).
Given two
different formulas $\varphi_i(\vek{x}_i, \vek{y}_i)$ and
$\varphi_j(\vek{x}_j, \vek{y}_j)$, we assume that
$\vek{y}_i\cap\vek{y}_j=\emptyset$, i.e.\ the sets of auxiliary
variables of encodings in different leaves are pairwise disjoint.

For a node \(v\in V\), let us denote \(\var{v}\) the set of main
variables from \(\vek{x}\) that appear in the leaves which can be
reached from \(v\) by a path.
In particular, \(\var{v}=\vek{x}_i\) for a leaf \(v\)
associated with \(\varphi(\vek{x}_i, \vek{y}_i)\).
We assume
that \(\var{\rho}=\vek{x}\), i.e.\ each variable \(x\in \vek{x}\) is
in some leaf.

Given a node \(v\in V\), let $f_v(\vek{x}_i)$ be the
function defined on the variables $\var{v}$ as follows.
If \(v\) is a leaf node associated with \(\varphi_i(\vek{x}_i,
   \vek{y}_i)\), then \(f_v(\vek{x}_i)\) is the function with encoding
\(\varphi_i(\vek{x}_i, \vek{y}_i)\).
If $v$ is a $\land$-node or a $\lor$-node, $f_v$ is the conjunction
or the disjunction, respectively, of the functions $f_u$ represented
by the inputs $u$ of $v$.

\begin{definition}[Backdoor Decomposable Monotone Circuit]
   \mlabel{def:bdmc} 
   Let \(\mathcal{C}\) be a base class of CNF encodings
containing every literal as a formula. A sentence in
   the language of \emph{backdoor decomposable monotone circuits with
      respect to base class \(\mathcal{C}\)}
   (\(\mathcal{C}\)-\emph{BDMC}) is a directed acyclic graph as
described above, where
   each leaf node is labeled with a CNF encoding from \(\mathcal{C}\) and
   each internal node is labeled with \(\land\) or \(\lor\) and can
   have arbitrarily many successors.
Moreover, the nodes labeled with \(\land\)
   satisfy the \emph{decomposability} property, which means that
for every $\land$-node $v=v_1\land\dots\land v_k$,
the sets of variables \(\var{v_1}, \dots, \var{v_k}\) are pairwise disjoint.
The function represented by the sentence is the function $f_\rho$
defined on the variables $\vek{x}=\var{\rho}$.
\end{definition}

We will omit prefix \(\mathcal{C}\) and write simply BDMC in case the
choice of a particular class of formulas \(\mathcal{C}\) is not
essential. The language of DNNFs is the class of those \(\mathcal{C}\)-BDMCs,
whose leaves are the literals on the input variables. Since a
decision node can be represented as a disjunction of two conjunctions
in which one of the conjuncts is a literal, we can also conclude that
\(\mathcal{C}\)-backdoor trees~\cite{SS08} form a subclass of
\(\mathcal{C}\)-BDMCs.

For the construction we describe in Section~\ref{sec:results}, we
consider the case when $\mathcal{C}$ is equal to the class of
PC encodings. This class admits a polynomial time satisfiability
test. However, \citet{BBCGK13} proved that the corresponding membership test whether a given formula
is PC is co-NP-complete. For this reason, when the
complexity of algorithms searching for a BDMC for a given function is
in consideration, a different suitable class of encodings with a polynomial time
membership test can be used, such as prime 2-CNF (which are PC) or (renamable) Horn
formulas (which are URC).

The function represented by a BDMC is described above by a recursion.
In order to relate this function to the encodings constructed later,
we describe the function using the following notion.
A \emph{minimal satisfying subtree} \(T\) of \(D\) is a rooted subtree
of \(D\) (also called out-arborescence) which has the following
properties:
\begin{itemize}
   \item The root \(\rho\) of \(D\) is also the root of \(T\).
   \item For each \(\land\)-node \(v\) in \(T\), all edges $(v,u)$ in $D$
         are in \(T\).
   \item For each \(\lor\)-node \(v\) in \(T\), exactly one of the edges $(v,u)$
         in $D$ is also in \(T\).
\end{itemize}

If \(\varphi_i(\vek{x}_i, \vek{y}_i)\) and \(\varphi_j(\vek{x}_j,
   \vek{y}_j)\) are formulas associated with two different leaves of
\(T\), then by decomposability of \(D\) we have that
\(\vek{x}_i\cap\vek{x}_j=\emptyset\).
We can observe that if \(\mathcal{T}\) denotes the set of all minimal
satisfying subtrees of \(D\), then we have
\begin{linenomath}
\begin{equation}
   \label{eq:phi-by-trees}
   f(\vek{x}) \equiv \bigvee_{T\in\mathcal{T}} \;
   \bigwedge_{\leaf{i}\in V(T)}
   (\exists \vek{y}_i) \varphi_i(\vek{x}_i, \vek{y}_i) \,.
\end{equation}
\end{linenomath}

A smooth BDMC is defined similarly to a smooth DNNF\@.

\begin{definition}[Smooth BDMC]
   \mlabel{def:smooth-bdmc}
   We say that a \(\mathcal{C}\)-BDMC \(D\) is \emph{smooth} if for
   every \(\lor\)-node \(v=v_1\lor\dots\lor v_k\) we have
   \(\var{v}=\var{v_1}=\dots=\var{v_k}\).
\end{definition}

If \(D\) is a smooth BDMC representing a function
\(f(\vek{x})\) and \(T\) is a minimal satisfying subtree of \(D\),
then for every \(x\in\vek{x}\) there is a leaf of \(T\) which is
associated with a formula \(\varphi_i(\vek{x}_i, \vek{y}_i)\)
such that \(x\in \vek{x}_i\).

The definition of a smooth BDMC restricts only the occurrences of the
main variables. The auxiliary variables of the encodings in the leaves
are local to the leaves and we assume that the sets of auxiliary
variables in the encodings associated with two differrent leaves are
disjoint.
\citet{D01a} showed that a DNNF can be transformed into a
smooth DNNF with a polynomial increase of size. The same approach can
be used to make an arbitrary \(\mathcal{C}\)-BDMC smooth.
This is one of the places, where we use the assumption
that every formula consisting of a single literal belongs to \(\mathcal{C}\).

Let us state the main result of this paper proven in
Section~\ref{sec:results} as a consequence of Theorem~\ref{thm:main}.

\begin{theorem}\label{thm:poly-translate}
   Let \(D\) be a smooth PC-BDMC representing
   a function \(f(\vek{x})\). Then we can construct in polynomial time
   a PC encoding of \(f(\vek{x})\).
\end{theorem}

Following~\citet{FM08a}, two languages \(\mathbf{L}_1\) and \(\mathbf{L}_2\) are called
\emph{polynomially equivalent} if any sentence in \(\mathbf{L}_1\) can be translated
in polynomial time into an equivalent sentence in \(\mathbf{L}_2\) and vice
versa. As a corollary of Theorem~\ref{thm:poly-translate} we get the
following.

\begin{corollary}\label{cor:poly-equiv}
   Languages of PC encodings and PC-BDMCs are polynomially equivalent.
\end{corollary}

The construction used to prove Theorem~\ref{thm:poly-translate} is described
later in Section~\ref{sec:results}.

\section{Relations to Other Target Compilation Languages}
\label{sec:relations}

Let us first recall the notion of succinctness introduced
by~\citet{GKPS95} and used later extensively by~\citet{DM02}.

\begin{definition}[Succinctness]
   Let \(\mathbf{L}_1\) and \(\mathbf{L}_2\) be two representation
   languages. We say that \(\mathbf{L}_1\) is \emph{at least as succinct as}
   \(\mathbf{L}_2\), iff there exists a polynomial \(p\) such that for
   every sentence \(\varphi\in\mathbf{L}_2\), there exists an
   equivalent sentence \(\psi\in\mathbf{L}_1\) where \(|\psi|\leq
      p(|\varphi|)\).
   We  say that \(\mathbf{L}_1\) is \emph{strictly more succinct} than
   \(\mathbf{L}_2\) if \(\mathbf{L}_1\) is at least as succinct as
   \(\mathbf{L}_2\) but \(\mathbf{L}_2\) is not at least as succinct
   as \(\mathbf{L}_1\).
\end{definition}

Corollary~\ref{cor:poly-equiv} implies that the languages of PC
encodings and PC-BDMCs are equally succinct. 
\citet{BCMS14,BCMS16} show examples of monotone CNF
formulas which have only exponentially bigger DNNFs. Given the fact
that every monotone CNF formula is PC, we get that PC encodings and
PC-BDMCs are strictly more succinct than DNNFs.

Let us relate PC-BDMCs and PC encodings to backdoor
trees introduced by~\citet{SS08} when we consider PC formulas as a base
class. Backdoor trees were introduced in the context of parameterized
SAT solving as an auxiliary data structure that allows to make SAT
solving of a given CNF formula easy. Given a base
class \(\mathcal{C}\), a \(\mathcal{C}\) backdoor tree \(T\) for a CNF
formula \(\varphi\) is defined as a decision tree on some of the variables
in \(\varphi\) which satisfies the following property: If \(\alpha\)
is a partial assignment specified by a path from the root of \(T\) to
a leaf, then \(\varphi(\alpha)\) (i.e., formula \(\varphi\) after we
apply partial assignment \(\alpha\)) belongs to class \(\mathcal{C}\).
In particular, all the formulas in the leaves are restrictions of the
same original formula. For proving a lower bound on the size
of a PC backdoor tree,
we remove this assumption, so we only require that the formula in
every leaf represents the restriction of the original function
according to the assignments on the path from the root to the leaf.
We will call this structure
\emph{generalized $\mathcal{C}$ backdoor tree}
and it is precisely the subclass of $\mathcal{C}$-BDMCs that satisfy
that the only gates allowed in the circuit part are decision gates
(disjunctions of two conjunctions), the directed graph in the circuit
part is an out-arborescence, and leaves are associated with
$\mathcal{C}$ encodings without auxiliary variables.

Let us point out another generalization of backdoor trees within the
framework of BDMCs obtained by including decomposable conjunctions.
This is a model that appears as an intermediate state in several
CNF to DNNF compilers, whose final output is
a Decision DNNF \cite[see e.g.,][]{LM17}.
From the compilation perspective, it thus makes sense
to consider a variant of BDMCs which only allows conjunctions and decision
nodes as inner nodes. One of the variants of the
lower bound below separates generalized PC backdoor trees with
decomposable conjunctions from generalized PC backdoor trees.

In this section, we present a family of boolean functions
which have PC-BDMCs of
polynomial size, but any generalized PC backdoor tree has
exponential size. For this purpose, we measure the sizes of
generalized PC
backdoor trees and PC-BDMCs in the same way, namely, we sum the sizes
of the formulas associated with the leaves with the number of the
edges in the circuit part.

For a given \(n\), let us define formula \(\psi_n'\) on \(2n\) variables
\(y_1, \dots, y_n, z_1, \dots, z_n\)
as follows.
\begin{linenomath}
\begin{equation*}
   \psi_n'=(\neg z_1\lor\dots\lor\neg z_n)\land\bigwedge_{i=1}^n(\neg
   y_i\lor z_i)
\end{equation*}
\end{linenomath}
It can be checked that \(\psi_n'\) has \(2^n\) implicates of form
\(C_I=\bigvee_{i\in I}\neg z_i\lor\bigvee_{i\not\in I}\neg y_i\) for every set
of indices \(I\subseteq\{1, \dots, n\}\) in addition to \(n\) prime
implicates \(\neg y_i\lor z_i\), \(i=1, \dots, n\).
Clause \(C_I\) is an implicate of \(\psi_n'\), because it can be
produced by resolving clause \(\neg z_1\lor ... \lor \neg z_n\)
with \(\neg y_i\vee z_i\) for \(i \not\in I\).
\citet[][Section 4.1]{KS20} argued
that formula
\begin{linenomath}
\begin{equation}
   \label{eq:psi}
\psi_n=\bigwedge_{C\in\psi'_n}(x\lor C)
\end{equation}
\end{linenomath}
where $x$ is a new variable has only one prime PC
representation which is the list of all its \(2^n+n\) prime
implicates. We use this property to show the
main result of this section. 

\begin{theorem}
   \label{thm:bdtree}
   PC-BDMCs (and thus also PC encodings) are strictly more succinct
   than generalized PC backdoor trees.
\end{theorem}
\begin{proof}
   For a given \(n\), let us consider three sets of variables: \(\vek{x}=\{x_{i,j}\mid i,
      j\in\{1, \dots, n\}\}\), \(\vek{y}=\{y_{i,j,k}\mid i,
      j\in\{1, \dots, n\}, k\in\{1, \dots, n-1\}\}\), and \(\vek{z}=\{z_{i,j}\mid i,
      j\in\{1, \dots, n\}\}\). For every \(i=1, \dots, n\), we
   introduce formulas
\begin{linenomath}
   \begin{align*}
      \gamma_i&=(\neg x_{i,1}\lor\dots\lor\neg x_{i,n})\\
      \delta_i&=(\neg z_{i,1}\lor\dots\lor\neg z_{i,n})\\
      &\land\bigwedge_{s=1}^n(\neg y_{i,s,1}\lor\dots\lor\neg y_{i, s, n-1}\lor z_{i, s})
   \end{align*}
   \end{linenomath}
   Let \(f(\vek{x}, \vek{y}, \vek{z})\) be the function represented by
   the disjunction of formulas
   \(\Gamma=\bigwedge_{i=1}^n\gamma_i\) and
   \(\Delta=\bigwedge_{j=1}^n\delta_j\).
   Since \(\Gamma\) is a conjunction of clauses on pairwise disjoint
   sets of variables, it is immediate that \(\Gamma\) is PC\@. Formulas
   \(\delta_i\), \(i=1, \dots, n\) are PC because each can be
   constructed by subsequently taking conjunction of PC formulas which share a
   single variable. This leads to a PC formula as shown
   by~\citet[][proof of Proposition 5]{BM12}.
   It follows that \(f\) has a small PC-BDMC
   in form of the disjunction of
   \(\Gamma\) and \(\Delta\). Let \(T\) be any generalized PC backdoor tree
   for \(f\) which
   has leaves associated with PC formulas of size less than $2^n$.
   We claim that every leaf in \(T\) has depth at least
   \(n\) and thus the number of leaves of \(T\) is at least \(2^n\).

   Let
   us consider any partial assignment
   \(\alpha\subseteq\lit{\vek{x}\cup\vek{y}\cup\vek{z}}\) of size \(n-1\).
   Since all prime implicates of \(\gamma_i\) and of \(\delta_j\) for all
   indices \(i\), \(j\) have length at least \(n\), we get that
   formulas \(\gamma_i(\alpha)\) and \(\delta_j(\alpha)\) are
   satisfiable. Moreover,
   there is a pair of indices \(p\), \(q\), such that
   \(\alpha\) does not assign a value to any variable in
   \(\gamma_{p}\) and \(\delta_{q}\).
Partial assignment \(\alpha\) can be
   extended to a partial assignment \(\alpha'\) such that
   \(f(\alpha')\equiv \gamma_{p}\lor\delta_{q}\). This is possible,
since $\gamma_i(\alpha)$, $\delta_j(\alpha)$ are pairwise independent and satisfiable.
Let us further extend \(\alpha'\) to satisfy all
   variables \(x_{p,1}, \dots, x_{p,n-1}\) and variables \(y_{q,s,t}\)
   for \(s=1, \dots, n\) and \(t=1, \dots, n-2\). 
In this way, we obtain a restriction of $f$ represented by the formula
\begin{linenomath}
   \[
      (\neg x_{p,n}\lor\neg z_{q,1}\lor\dots\lor\neg z_{q,n})
      \land\bigwedge_{s=1}^n(\neg x_{p,n}\lor\neg
      y_{q,s,n-1}\lor z_{q,s})
   \]
   \end{linenomath}
   which has the same structure as \(\psi_n\) defined
   by~\eqref{eq:psi}.
This is a contradiction with the choice of $\alpha$, since
any PC formula equivalent to~\eqref{eq:psi} has size at least \(2^n\).
It follows that every leaf in \(T\) has depth at least \(n\).
\end{proof}

The separation presented in Theorem~\ref{thm:bdtree} is based on
a function defined by a disjunction of suitably chosen formulas
which is not allowed in generalized backdoor trees.
A similar separation could be achieved with a
function \(g\) defined by a conjunction, namely,
by the formula \(\bigwedge_{i=1}^n(\gamma_i\lor\delta'_i)\)
where \(\delta'_i=(\neg z_{i,1}\lor\dots\lor\neg
   z_{i,n})\land\bigwedge_{j=1}^n(\neg y_{i,1}\lor z_{i,j})\).

\section{Queries and Transformations}

In this section we shall look at PC encodings and PC-BDMCs as target
compilation languages. We shall demonstrate that both these languages
have the same properties as DNNFs when it comes to answering queries and
transformations described by~\citet{DM02}.

Let us first look at queries. \citet{DM02} consider \textbf{CO}
(consistency), \textbf{VA} (validity), \textbf{CE} (clausal
entailment), \textbf{EQ} (equivalence), \textbf{SE} (sentential
entailment), \textbf{IM} (implicant), \textbf{CT} (model counting),
and \textbf{ME} (model enumeration). Out of these, \textbf{CO},
\textbf{CE} and \textbf{ME} can be done in polynomial time on a DNNF,
the remaining ones cannot be performed in polynomial time unless P is
equal to NP\@. Since DNNFs form a special case of PC-BDMCs, we have
that also for PC encodings and PC-BDMCs, answering queries
\textbf{VA}, \textbf{EQ}, \textbf{SE}, \textbf{IM}, and \textbf{CT} is
hard. Consistency checking \textbf{CO} and clausal entailment
\textbf{CE} can be performed on a PC encoding by unit propagation.
PC encodings also satisfy \textbf{ME}, because the models of a PC
encoding can be enumerated with polynomial delay by a simple backtrack
procedure considering the fact that PC encodings are closed under the
application of a partial assignment. 
In particular, to enumerate the models of a PC encoding \(\varphi\), first
check if it is satisfiable and if so, pick an unassigned variable of
\(\varphi\) and recursively enumerate the models of \(\varphi(x)\) and
\(\varphi(\neg x)\) which originate from \(\varphi\) by satisfying
literals \(x\) and \(\neg x\) respectively.
By Corollary~\ref{cor:poly-equiv} we have that PC-BDMCs and PC encodings are polynomially equivalent and thus
they have
the same properties with respect to query answering.

\citet{DM02} consider the following transformations on compilation
languages: \textbf{CD} (conditioning), \textbf{FO} (forgetting),
\textbf{SFO} (singleton forgetting), \textbf{\(\land\)C}
(conjunction), \textbf{\(\land\)BC} (bounded conjunction),
\textbf{\(\lor\)C} (disjunction), \textbf{\(\lor\)BC} (bounded
disjunction), and \textbf{\(\neg\)C} (negation). Unless P is equal to
NP, DNNFs do not allow polytime \textbf{\(\land\)C},
\textbf{\(\land\)BC}, and \textbf{\(\neg\)C} and since DNNFs are a
special case of PC-BDMCs, this is also the case for PC-BDMCs. It is
well-known that these operations are not efficient also for PC
encodings under the same assumption. The rest of transformations can
be done in polynomial time on DNNFs. Let us look at these
transformations on PC-BDMCs and PC encodings. \textbf{CD} can be done
in polynomial time on PC encodings since partial assignment preserves
propagation completeness. Both \textbf{FO} and \textbf{SFO} are
trivial on a PC encoding, we just move the variables to be forgotten
from the set of main variables to the set of auxiliary variables. Note
that PC encodings also allow to forget a single auxiliary variable by
means of Davis Putnam resolution. Both \textbf{\(\lor\)C} and
\textbf{\(\lor\)BC} can be done on PC-BDMCs in the same way as on
DNNFs, just connect the roots of the input PC-BDMCs with a disjunction
gate. This transformation is not so trivial on PC encodings. By
Theorem~\ref{thm:poly-translate} we have that a PC-BDMC can be
translated into a PC encoding in polynomial time and thus a
disjunction of PC encodings can be transformed back into a PC
encoding.

\section{Extended Implicational Dual Rail Encoding}
\label{sec:dual-rail} 

We use the well-known dual rail encoding of partial
assignments~\cite[see e.g.,][]{BBIMM18,BBBCS87,IMM17,MFSO97,MIBMB19}
to simulate unit propagation in a general CNF formula in the same way
as~\citet{BKNW09,BJM12,KS20}. We use the form of the encoding with
a special variable representing the fact that contradiction was not derived
and extend it with clauses
which make the encoding propagation complete if the input CNF formula is PC\@.

Let us introduce for
every $l \in \lit{\vek{x}}$ a \emph{meta-variable} $\litVar{l}$.
In addition, we use special meta-variable $\litVar{\top}$
intended to represent the value of a formula
in a way suitable for propagating into the circuit part of BDMC\@.
For this purpose, we implement deriving a contradiction
as deriving the negative literal $\neg \litVar{\top}$.
The set of the meta-variables
corresponding to a vector of variables $\vek{x}$ will be denoted
\begin{linenomath}
\begin{equation*}
   \meta{\vek{x}}=\{\litVar{l}\mid l\in\lit{\vek{x}} \cup
   \{\top\}\}\,.
\end{equation*}
\end{linenomath}

For notational convenience, we extend this notation also to sets
of literals that are meant as a conjunction, especially to partial
assignments. If \(\alpha\subseteq\lit{\vek{x}}\) is a set of literals,
then \(\litVar{\alpha}=\{\litVar{l}\mid l\in\alpha\}\) denotes the set of meta-variables associated
with the literals in \(\alpha\). If \(\litVar{\alpha}\) is used in a formula such as
\(\psi\land\litVar{\alpha}\), we identify this set of literals
with the conjunction of them, similarly to the interpretation of
\(\alpha\) in \(\phi\land\alpha\).

\begin{definition}[Extended implicational dual rail encoding]\mlabel{def:drenc}
   Let \(\varphi(\vek{x})\) be an arbitrary CNF formula. The
   \emph{extended implicational dual rail encoding} of \(\varphi\) is a formula
   on meta-variables \(\vek{z}=\meta{\vek{x}}\)
   denoted \(\xdrenc{\varphi}{\vek{z}}\) and defined as follows.
If $\phi$ contains the empty clause, then
$\xdrenc{\varphi}{\vek{z}}=\neg\litVar{\top}$.
Otherwise, we set
\begin{linenomath}
   \begin{equation}
      \label{eq:def-drenc-impl}
      \begin{aligned}
      \xdrenc{\varphi}{\vek{z}}&=\bigwedge_{C\in\varphi}\bigwedge_{l\in
         C}\left(\bigwedge_{e\in C\setminus\{l\}}\litVar{\neg
            e}\to\litVar{l}\right)\\
     & \land\bigwedge_{x\in\vek{x}}(\litVar{x}\land\litVar{\neg
         x}\to
      \neg\litVar{\top})\\
      &\land
\bigwedge_{l\in\lit{\vek{x}}}\left(\litVar{\top} \vee \litVar{l}\right)
\wedge
\bigwedge_{x\in \vek{x}}\left(\litVar{x}\vee \litVar{\neg x}\right)\,.
   \end{aligned}
   \end{equation}
\end{linenomath}
\end{definition}

A subset of extended implicational dual rail encoding is used in
the first part of the proof of Theorem 1{} by~\citet{BKNW09} for a
similar purpose as in this paper. A similar encoding is used also
by~\citet{BJM12} as a part of a larger formula. The proof of the
following lemma is omitted, since it is a straightforward extension
of the properties of the encodings used in the two papers cited above.

\begin{lemma}\mlabel{lem:dr-enc}
   Let \(\varphi(\vek{x})\) be a CNF formula not containing the empty clause,
   let \(\alpha\subseteq\lit{\vek{x}}\), and assume
   \(\vek{z}=\meta{\vek{x}}\). Then
   \begin{linenomath}
   \begin{equation*}
      \varphi\land \alpha\vdash_1 \bot
      \iff
      \xdrenc{\varphi}{\vek{z}}\land\litVar{\alpha}\vdash_1 \neg\litVar{\top}
   \end{equation*}
\end{linenomath}
and if \(\varphi\land\alpha\not\vdash_1\bot\), then for every \(l\in\lit{\vek{x}}\) we have
\begin{linenomath}
   \begin{equation*}
      \varphi\land \alpha\vdash_1 l
      \iff
      \xdrenc{\varphi}{\vek{z}}\land\litVar{\alpha}\vdash_1 \litVar{l}\,.
   \end{equation*}
\end{linenomath}
\end{lemma}

The clauses of form \(\litVar{\top} \vee \litVar{l}\)
imply that once \(\neg\litVar{\top}\) is derived by unit propagation,
which rules out the possibility that \(\varphi\) is consistent with a
given partial assignment, all
meta-variables \(\litVar{l}\), \(l\in\lit{\vek{x}}\) are derived as
well. The clauses $\litVar{x}\vee \litVar{\neg x}$ guarantee that
every satisfying assignment can be extended to a satisfying assignment
of a formula obtained by adding the clauses \(l\to\litVar{l}\), \(l\in\lit{\vek{x}}\).
We shall show that extended implicational dual rail encoding of a PC
formula is PC as well. We will need the following auxiliary lemma.

\begin{lemma}
   \mlabel{lem:xdrenc-models}
Let \(\varphi(\vek{x})\) be a PC formula, $\beta \subseteq \lit{\vek{x}}$
and $\beta' = \litVar{\beta} \cup \{\neg \litVar{\neg l} \mid l \in \beta\}$.
If $\varphi\land\beta\not\vdash_1\bot$,
then $\xdrenc{\varphi}{\meta{\vek{x}}}\land\beta' \wedge \litVar{\top}$
is satisfiable.
\end{lemma}

\begin{proof}
Since $\varphi$ is PC, there is a satisfying assignment
\(\assign{b}\) of $\phi(\vek{x}) \wedge \beta$.
Define an assignment $\assign{b'}$ as $\assign{b'}(\litVar{l}) = \assign{b}(l)$
for all $l \in \lit{\vek{x}}$
and $\assign{b'}(\litVar{\top}) = 1$. One can verify that this assignment satisfies
$\xdrenc{\varphi}{\meta{\vek{x}}} \wedge \beta' \wedge \litVar{\top}$.
\end{proof}

We are now ready to show that the extended implicational dual rail
encoding of a PC fomula is itself PC\@.

\begin{lemma}\mlabel{lem:xdrenc-urc-pc}
   If \(\varphi(\vek{x})\) is a PC formula, then
   \(\xdrenc{\varphi}{\meta{\vek{x}}}\) is a PC formula.
\end{lemma}
\begin{proof}
   The proposition is trivially true if \(\varphi\) contains the empty
   clause, so let us assume this is not the case.
   Let \(\alpha\subseteq\lit{\meta{\vek{x}}}\) be a partial assignment and
   let us assume that \(\xdrenc{\varphi}{\meta{\vek{x}}}\land\alpha\not\vdash_1\bot\).
   Assume moreover that \(\alpha\) is closed under unit propagation in
   \(\xdrenc{\varphi}{\meta{\vek{x}}}\) and consider a literal
\(l\in\lit{\meta{\vek{x}}}\) such that \(\neg l\not\in\alpha\) or, equivalently,
$\xdrenc{\varphi}{\meta{\vek{x}}} \wedge \alpha \not\vdash_1 \neg l$.
We show that then \(\xdrenc{\varphi}{\meta{\vek{x}}}\land\alpha\land l\)
is satisfiable.

If \(l=\neg\litVar{\top}\), then $\litVar{\top}\not\in \alpha$
by assumption.
   Due to clauses \(\litVar{\top}\vee\litVar{g}\)
   in \(\xdrenc{\varphi}{\meta{\vek{x}}}\), we get
   \(\neg\litVar{g}\not\in\alpha\) for every \(g\in\lit{\vek{x}}\). It
   follows that the assignment which sets the value of
   \(\litVar{\top}\) to \(0\) and the values of all remaining variables in
   \(\meta{\vek{x}}\) to \(1\) satisfies
   all clauses of
   \(\xdrenc{\varphi}{\meta{\vek{x}}}\land\alpha\land\neg\litVar{\top}\). 

   If \(l=\litVar{\top}\), then \(\neg\litVar{\top}\not\in\alpha\) by assumption.
Let \(\beta\subseteq\lit{\vek{x}}\) be defined as
\(\beta=\{l \in \lit{\vek{x}} \mid \litVar{l}\in\alpha\}\).
Since $\litVar{\beta} \subseteq \alpha$, we have, by assumption,
\(\xdrenc{\varphi}{\meta{\vek{x}}}\land\litVar{\beta}\not\vdash_1\neg\litVar{\top}\) and thus by Lemma~\ref{lem:dr-enc}
we have \(\varphi\land\beta\not\vdash_1\bot\).
Using Lemma~\ref{lem:xdrenc-models} and the definition of $\beta'$ in it, we obtain that
$\xdrenc{\varphi}{\meta{\vek{x}}} \wedge \beta' \wedge \litVar{\top}$
is satisfiable.
Since $\mathrm{DR}^+$ contains the clauses \(\litVar{x}\lor\litVar{\neg x}\),
we have for every literal \(g\in\lit{\vek{x}}\)
that if \(\neg\litVar{g}\in\alpha\), then \(\litVar{\neg g}\in\alpha\),
$\neg g \in \beta$, and $\neg\litVar{g}\in\beta'$. Together with the
assumption $\neg \litVar{\top} \not\in \alpha$,
we obtain $\alpha \subseteq \beta' \wedge \litVar{\top}$ implying that
$\xdrenc{\varphi}{\meta{\vek{x}}} \wedge \alpha \wedge \litVar{\top}$
is satisfiable.

Assume \(l=\litVar{g}\) where \(g\in\lit{\vek{x}}\), thus
   $\neg\litVar{g} \not\in \alpha$ by assumption.
If $\litVar{\top} \not\in \alpha$, the satisfying assignment described
above for $l = \neg \litVar{\top}$ satisfies $\litVar{g}$ and $\alpha$.
If \(\litVar{\top}\in\alpha\), then $\neg\litVar{g} \not\in
   \alpha$ implies $\litVar{\neg g} \not\in \alpha$ and we also
   have $\neg\litVar{\top} \not\in \alpha$.
If $\beta$ is defined as above, we have $\litVar{\beta} \subseteq \alpha$ and, hence,
\(\xdrenc{\varphi}{\meta{\vek{x}}}\land\litVar{\beta}\not\vdash_1 \litVar{\neg g}\).
Since \(\xdrenc{\varphi}{\meta{\vek{x}}}\) contains clause
\(\litVar{\top}\lor\litVar{\neg g}\), we have that 
\(\xdrenc{\varphi}{\meta{\vek{x}}}\land\litVar{\beta}\not\vdash_1 \neg \litVar{\top}\).
Lemma~\ref{lem:dr-enc} implies
\(\varphi(\vek{x})\land\beta\not\vdash_1\bot\) and
\(\varphi(\vek{x})\land\beta\not\vdash_1 \neg g\). Since
\(\varphi\) is PC, we have that \(\varphi\land\beta\land
   g\not\vdash_1\bot\).
It follows by Lemma~\ref{lem:xdrenc-models} used on partial assignment
\(\beta\land g\) that the formula
$\xdrenc{\varphi}{\meta{\vek{x}}}\land\beta' \wedge
\litVar{g} \wedge \litVar{\top}$ is satisfiable.
It follows that
$\xdrenc{\varphi}{\meta{\vek{x}}}\land\alpha \wedge \litVar{g}$
is satisfiable, since we have
$\alpha \subseteq \beta' \wedge \litVar{\top}$ as in the
previous case.

   At last, assume
   \(l=\neg\litVar{\neg g}\) where \(g\in\lit{\vek{x}}\), thus
   $\litVar{\neg g} \not\in \alpha$ by assumption.
   It follows that $\neg\litVar{\top} \not\in \alpha$, since
   otherwise \(\litVar{\neg g}\) would be derived using the clause
   \(\litVar{\top}\vee\litVar{\neg g}\). If we define \(\beta\) as above,
   then we can use Lemma~\ref{lem:dr-enc} and the propagation
   completeness of \(\varphi\) as in the case \(l=\litVar{g}\) to conclude
   \(\varphi\land\beta\land g\not\vdash_1\bot\). Using
   Lemma~\ref{lem:xdrenc-models} on partial assignment \(\beta\land g\)
we get that
$\xdrenc{\varphi}{\meta{\vek{x}}}\land\beta' \wedge
\neg \litVar{\neg g} \wedge \litVar{\top}$ is satisfiable.
It follows that
$\xdrenc{\varphi}{\meta{\vek{x}}}\land\alpha \wedge \neg \litVar{\neg g}$
is satisfiable, since we have
$\alpha \subseteq \beta' \wedge \litVar{\top}$ as above.
\end{proof}

We use extended implicational dual rail encodings of the formulas in
all the leaves as part
of the encoding of a BDMC\@. In order to make the propagation in them
independent, the extended implicational dual rail encoding in $\leaf{i}$ uses
the following set of meta-variables with indices in place of $\vek{z}$
\begin{linenomath}
\begin{equation*}
   \metaI{i}{\vek{x}}=\{\metaVar{i}{l}\mid l\in\lit{\vek{x}} \cup \{\top\}\}\,.
\end{equation*}
\end{linenomath}
This is the reason for including the set of variables
\(\vek{z}\) as part of the notation \(\xdrenc{\varphi}{\vek{z}}\). 
Encodings \(\xdrenc{\varphi_i}{\vek{z}_i}\), \(i=1, \dots, \ell\) are
connected to the remaining parts by identifying the variable representing
leaf node \(\leaf{i}\) with the variable $\metaVar{i}{\top}$ and the
consistency with the input variables is guaranteed by adding the clauses
$l \to \metaVar{i}{l}$ for all $l \in \lit{\vek{x}_i}$~\citep[see also][]{BJM12}.
Note that these clauses
imply $\metaVar{i}{x} \vee \metaVar{i}{\neg x}$ for each $x \in \vek{x}_i$.
The implied clauses are explicitly present in
\(\xdrenc{\varphi_i}{\vek{z}_i}\)
 also for $x \in \vek{y}_i$.

\section{PC Encoding of a PC-BDMC}
\label{sec:results}

In this section we present the encoding promised in the introduction.
To this end, let us fix a PC-BDMC \(D\) which represents a function \(f(\vek{x})\) on
variables \(\vek{x}=(x_1, \dots, x_n)\). Let \(V\) denote the set of
nodes in \(D\) and let \(\rho\) denote the root of \(D\). Let us
assume that \(D\) has \(\ell\) leaves which are labeled with PC
encodings \(\varphi_i(\vek{x}_i, \vek{y}_i)\), \(i=1, \dots, \ell\).

\subsection{Separators}
\label{ssec:separators}

The construction of the PC encoding relies on the
notion of separators introduced by~\citet{KS19} and briefly described below.
Construction of separators can require to include
a polynomial number of new nodes into $D$.
For every \(i=1, \dots, n\), let us denote \(H_i\) the set of
nodes \(v\) of \(D\) for which \(x_i\in\var{v}\). Let \(D_i\) be the
subgraph of \(D\) induced by the vertices in \(H_i\). We say that a
set of nodes \(S\subseteq H_i\) is a \emph{separator} in \(D_i\), if
every path in \(D_i\) from the root to a leaf contains precisely one
node from \(S\). We say that \(D\) can be \emph{covered by
   separators}, if for each \(i=1, \dots, n\) there is a collection of
separators \(\mathcal{S}_i\) in \(D_i\), such that the union of
\(S\in\mathcal{S}_i\) is \(H_i\). Not every BDMC can be covered by
separators, but every BDMC can be modified in polynomial time into an
equivalent one which can be covered by separators~\citep[for more
details, see][]{KS19}.

For the purpose of the construction, let us assume that \(D\) is a smooth BDMC covered by separators
and let us denote \(\mathcal{S}_i\) a set of separators which covers
\(D_i\). Let us further denote
\(\mathcal{S}=\bigcup_{i=1}^n\mathcal{S}_i\) the set of all separators
considered for covering \(D\). It follows by smoothness of \(D\) that if
\(T\) is a minimal satisfying subtree of \(D\), then for every \(i=1,
   \dots, n\), the intersection of \(T\) and \(D_i\) is a path
from the root \(\rho\) to a leaf \(v\) in \(D_i\), i.e.\ a leaf \(v\)
satisfying \(x_i\in\var{v}\). It follows that \(T\) must intersect each
separator \(S\in\mathcal{S}\) in exactly one node.

\subsection{Encoding}

We use the extended implicational dual rail encodings $\xdrenc{\phi_i}{\vek{z}_i}$
of the formulas $\varphi_i(\vek{x}_i, \vek{y}_i)$ associated with
the leaves of \(D\), where \(\vek{z}_i = \metaI{i}{\vek{x}_i \cup
      \vek{y}_i}\) is the set of meta-variables specific to the
formula $\phi_i$. The (disjoint) union of these sets of variables is denoted
\(\vek{z}=\bigcup_{i=1}^\ell \vek{z}_i\).

We associate a variable \(v\)
with every node \(v\in V\). For a leaf \(v\in V\) labeled with
\(\varphi_i(\vek{x}_i, \vek{y}_i)\) for some \(i\in \{1, \dots, \ell\}\),
the variable \(v\) is identified with \(\metaVar{i}{\top}\). This identification
is understood as follows. The variable actually used in the encoding
is $\metaVar{i}{\top}$. The variable \(v\) corresponding to the leaf is used
for simplicity if we refer to its role within the monotone circuit part
of $D$ while \(\metaVar{i}{\top}\) is used when referring to the role of the
leaf within the extended implicational dual-rail encoding of
\(\varphi_i(\vek{x}_i, \vek{y}_i)\) which is a part of the encoding.

The set of variables associated with the inner
nodes of \(D\) (i.e.\ the \(\land\)-nodes and \(\lor\)-nodes)
will be denoted \(\vek{v}\). As explained above, a leaf is represented by a variable
$\metaVar{i}{\top}$ which belongs to $\vek{z}$.
Altogether, the encoding described in this section uses three
kinds of variables: \(\vek{x}\), \(\vek{z}\), and \(\vek{v}\).

\begin{table*}[t]
   \centering
      \setlength{\tabcolsep}{12pt}
      \begin{tabular}{l l l}
         \toprule
         \multicolumn{1}{c}{group} & clause & condition\\
         \midrule
         \manuallabel{N1} 
         & \(v\to v_1\vee\dots\vee v_k\) &\(v=v_1\vee\dots\vee v_k\)\\
         \manuallabel{N2} 
         & \(v\to v_i\) &\(v=v_1\land\dots\land v_k\), \(i=1, \dots, k\)\\
         \manuallabel{N3} 
         & \(v\to p_1\vee\dots\vee p_k\) &\(v\in V\) has incoming edges from
         \(p_1, \dots, p_k\)\\
         \manuallabel{N6} 
         & \(\exonerep{S}\) & $S \in \mathcal{S}$\\
	 \manuallabel{R} & $\rho$ & $\rho$ \mbox{\ is the root}\\
         \manuallabel{E1} 
         & \(\xdrenc{\varphi_i}{\vek{z}_i}\) & \(i=1, \dots, \ell\),
           $\vek{z}_i=\metaI{i}{\vek{x}_i \cup \vek{y}_i}$\\
         \manuallabel{E2} 
         & \(l\to\metaVar{i}{l}\) & \(l\in\lit{\vek{x}_i}\), \(i=1,
            \dots, \ell\)\\
         \manuallabel{E3} 
         & \(\left(\bigwedge_{i\in\range{l}}\metaVar{i}{l}\right)\to l\)&\(l\in\lit{\vek{x}}\)\\
         \bottomrule
      \end{tabular}
   \caption{List of clauses of the encoding \(\encPC(\vek{x}, \vek{z}, \vek{v})\).
      \(\mathcal{S}=\bigcup_{i=1}^n\mathcal{S}_i\) denotes a fixed collection of separators 
      which covers \(D\) and
 \(\exonerep{S}\) denotes a suitable PC encoding of the exactly-one constraint on the variables in \(S\).
      By construction, we have \(v=\metaVar{i}{\top}\)
      for a leaf node \(v\) associated with formula
      \(\varphi_i(\vek{x}_i, \vek{y}_i)\).}\label{tab:clauses}
\end{table*}

Consider the list of clauses in Table~\ref{tab:clauses}. 
Clauses~\ref{N1}--\ref{N3} are the same as introduced by~\citet{AGMES16}
and~\cite{KS19}. Clauses~\ref{N6} are the same as introduced by~\citet{KS19} and a
special case of them was also used by~\citet{AGMES16}. Clauses N4 used in
the cited constructions are not needed for encodings of BDMC\@. We use
exactly-one constraints in group~\ref{N6}, similarly to~\citet{KS19},
if we would use the at-most-one constraints instead, we would obtain a
URC encoding provided the encodings associated with leaves are URC\@.
Let us point out that group~\ref{N6} contains a redundant unit clause
$\rho$, since one of the separators in ${\cal S}$ is $\{\rho\}$.

Let us denote \(\encPC(\vek{x}, \vek{z}, \vek{v})\) the encoding
consisting of the clauses in Table~\ref{tab:clauses}. The encoding is
clearly constructible in polynomial time and in particular, it has a
polynomial size as well. The constructed encoding proves
Theorem~\ref{thm:poly-translate} due to the following.

\begin{theorem}\mlabel{thm:main}
   Let \(D\) be a smooth PC-BDMC covered by separators and representing
   a function \(f(\vek{x})\).
   Then \(\encPC(\vek{x}, \vek{z}, \vek{v})\) is a PC
   encoding of \(f(\vek{x})\).
\end{theorem}

\begin{proof}
Combine Proposition \ref{prop:correctness} and Proposition \ref{prop:main-PC}
below.
\end{proof}

Let us close the section with an asymptotic estimate of the size of
encoding \(\encPC\). We will assume that the exactly one constraints
(clauses \ref{N6}) are represented with a linear size encoding
described by~\citet{KS19}. The encoding then contains \(O(ns+m)\)
variables and \(O(ns+e+r)\) clauses where \(n\) denotes the number of
input variables, \(s\) the number of nodes of \(D\), \(e\) the number
of edges of \(D\), \(m\) the total number of variables in the formulas
associated with the leaves, and \(r\) the total length of these
formulas (the sum of the lengths of all the clauses).

\subsection{Correctness of the encoding}
\label{sec:subtree} 

In this section we show that the encoding \(\encPC(\vek{x}, \vek{z}, \vek{v})\)
introduced in Section~\ref{sec:results} is correct in that it is a CNF encoding
of the function \(f(\vek{x})\) represented by a given BDMC \(D\).
One of the needed implications is proven in the following.

\begin{lemma}
   \mlabel{lem:correct-1}
   Let \(D\) be a smooth BDMC covered by separators
   which represents a function \(f(\vek{x})\). Let
   \(\assign{a}\) be a model of this function. Then there are
   assignments \(\assign{b}\) and \(\assign{c}\) of values to
   \(\vek{z}\) and \(\vek{v}\) respectively that satisfy all
   clauses introduced in Table~\ref{tab:clauses}.
\end{lemma}
\begin{proof}
Since $f(\assign{a})=1$, there is a minimal satisfying
subtree \(T\) of \(D\) which is consistent with \(\assign{a}\)
by~\eqref{eq:phi-by-trees}.
For every node \(v\in \vek{v}\) set \(\assign{c}(v)=1\) if and
only if \(v\) is in \(T\). A variable $v$ representing $\leaf{i}$
is identified with \(\metaVar{i}{\top}\) in $\vek{z}$. Consistently
with the above, set \(\assign{b}(v)=\assign{b}(\metaVar{i}{\top})=1\)
if and only if $\leaf{i}$ is in $T$.

For every formula \(\varphi_i(\vek{x}_i, \vek{y}_i)\) proceed as follows.
If \(\assign{b}(\metaVar{i}{\top})=0\), set
\(\assign{b}(\metaVar{i}{l})=1\) for
every literal \(l\in\lit{\vek{x}_i\cup\vek{y}_i}\).
If \(\assign{b}(\metaVar{i}{\top})=1\), denote \(\assign{a}_i\)
the part of \(\assign{a}\) on variables \(\vek{x}_i\).
Since \(T\) is consistent with \(\assign{a}\) and contains $\leaf{i}$,
   we get that \(\varphi_i(\vek{a}_i, \vek{y}_i)\)
   has a satisfying assignment
   \(\assign{d}_i\) to the variables \(\vek{y}_i\). Using this, set
   \(\assign{b}(\metaVar{i}{l})=\assign{a}(l)\) for every
   \(l\in\lit{\vek{x}_i}\) and
   \(\assign{b}(\metaVar{i}{l})=\assign{d}_i(l)\) for every
   \(l\in\lit{\vek{y}_i}\).

It is not hard to check that the clauses of groups~\ref{N1}--\ref{N3}, \ref{E1},
and~\ref{R} are satisfied by \(\assign{b}\cup\assign{c}\), if the leaves
are interpreted as elements of $V$ or $\vek{z}$ as appropriate for each group.
Since \(D\) is decomposable and smooth, the
intersection of \(T\) with a particular subgraph \(D_i\) is a path
from the root \(\rho\) to a leaf.
Such a path has exactly one node in common with every separator
   \(S\in\mathcal{S}_i\), so the clauses~\ref{N6} are satisfied.

Since the assignments $\assign{a}_i$ are derived from $\assign{a}$,
the clauses~\ref{E2} are satisfied by $\assign{a}\cup\assign{b}$.
Since \(D\) is smooth, for every literal \(l\) there is an index \(i\in\range{l}\)
such that \(T\) contains \(\leaf{i}\). It follows that the conjunction
\(\bigwedge_{i\in\range{l}}\metaVar{i}{l}\) is satisfied only
if $l$ is satisfied by $\assign{a}$, and thus the clauses in group~\ref{E3}
are satisfied.
\end{proof}

The following lemma shows the opposite implication.

\begin{lemma}
   \mlabel{lem:correct-2}
   Let \(D\) be a smooth BDMC covered by separators
which represents a function \(f(\vek{x})\) and
   let \(\assign{a}\), \(\assign{b}\), and \(\assign{c}\) be
   assignments to \(\vek{x}\), \(\vek{z}\), and \(\vek{v}\)
   respectively. If \(\encPC(\assign{a}, \assign{b},
      \assign{c})\) is satisfied, then \(\vek{a}\) is a model of
   \(f(\vek{x})\).
\end{lemma}
\begin{proof}
   Let us denote \(D'\) the subgraph of \(D\) which is induced by the
   inner nodes \(v\) for which \(\assign{c}(v)=1\) and by the leaves
   \(v\) for which \(\assign{b}(v)=1\). We show that \(D'\)
   contains a minimal satisfying subtree \(T\) of \(D\)
   consistent with \(\assign{a}\) by the construction consisting in repeated
application of the following steps starting with $T=\{\rho\}$. Note that
$\rho$ belongs to \(D'\), since the unit clause \(\rho\) is in the encoding.
The construction stops, when all the leaves of $T$ are leaves of $D$.
   \begin{itemize}
      \item If \(v=v_1\lor\dots\lor v_k\) is a leaf of \(T\), the clauses of
         group~\ref{N1} imply that at least one of the successors \(v_i\)
         belongs to \(D'\) as well. Include one of such successors into \(T\).
      \item If \(v=v_1\land\dots\land v_k\) is a leaf of \(T\), the clauses
         of group~\ref{N2} imply that the successors \(v_1, \dots, v_k\)
         belong to \(D'\). Include all the successors of $v$ into \(T\).
   \end{itemize}

If \(v\) is a leaf of $D$ associated with \(\varphi_i(\vek{x}_i, \vek{y}_i)\)
and included into \(T\), then \(\assign{b}(\metaVar{i}{\top})=1\).
Moreover, the assignment \(\assign{b}\)
satisfies \(\xdrenc{\varphi_i}{\vek{z}_i}\) represented by the clauses of
group~\ref{E1}. In particular, we have
$\neg \metaVar{i}{x} \vee \neg \metaVar{i}{\neg x}$
for every $x \in \vek{x}_i \cup \vek{y}_i$.
Together with the clauses of group~\ref{E2} we get that the assignment
$\assign{b}$ of the meta-variables that makes
$\xdrenc{\varphi_i}{\vek{z}_i}$ satisfied, is the dual-rail encoding
of the corresponding part of the assignment \(\assign{a}\). Consequently,
$\assign{a}$ satisfies \(\varphi_i(\vek{x}_i, \vek{y}_i)\).
It follows that $T$ is a minimal satisfying subtree of $D$ consistent
with $\assign{a}$ and~\eqref{eq:phi-by-trees} implies that \(\assign{a}\)
is a model of \(f(\vek{x})\).
\end{proof}

\begin{proposition} \mlabel{prop:correctness}
Let \(D\) be a smooth BDMC covered by separators
representing a function \(f(\vek{x})\). Then
   \(\encPC\) is a CNF encoding of \(f(\vek{x})\).
\end{proposition}

\begin{proof}
Follows from lemmas~\ref{lem:correct-1} and~\ref{lem:correct-2}.
\end{proof}

\subsection{Propagation Completeness}
\label{sec:urc-pc}

Let us fix a smooth PC-BDMC \(D\) together with a cover by separators
and let \(\encPC(\vek{x}, \vek{z}, \vek{v})\) be the encoding
constructed for \(D\). We show in this section that \(\encPC\) is a PC
encoding. Since a PC encoding should satisfy the propagation completeness
on all variables, the proof is based on the structure of
the formula and is independent of the fact that the formula is an encoding
of the function represented by $D$ on the main variables
proven in Section~\ref{sec:subtree}.

The monotone circuit part of a BDMC is the same as in a DNNF and has
the same encoding as in the encoding of DNNF presented by~\citet{KS19}.
Let \(\psi_p\)
be the formula formed by clauses of groups \ref{N1}--\ref{N3},
\ref{N6}, and $\rho$
defined over variables which correspond to all nodes \(V\) of \(D\),
i.e.\ \(\vek{v}\) and the leaves. The formula denoted
$\psi_p$ by~\citet{KS19} differs from this formula
in the interpretation of the leaves, however, it has the
same structure otherwise.
Recall that if \(v\) is a leaf of BDMC \(D\) labeled with formula
\(\varphi_i(\vek{x}_i, \vek{y}_i)\), then
\(v\) is identified with \(\metaVar{i}{\top}\).

\begin{lemma}
   \mlabel{lem:bdmc-notbot-sat} 
   Let \(\alpha\subseteq\lit{V}\) be a partial assignment, such that
   \(\psi_p\land\alpha\not\vdash_1\bot\). Let \(v_0\in V\) be a node
   such that \(\psi_p\land\alpha\not\vdash_1\neg v_0\). Then
   \(\psi_p\land\alpha\land v_0\) is satisfiable.
\end{lemma}
\begin{proof}
Define \(V'=\{v\in V\mid\psi_p\land\alpha\not\vdash_1\neg v\}\).
We find a minimal satisfying tree \(T\) of \(D\)
within $V'$ in the same way as used by \citet{KS19}
in the proof of Lemma~7.3. In both cases, the structure of
   the separators and, hence, clauses~\ref{N6}, depends
   on the assignment of variables \(\var{v}\) to each node. For a
   smooth DNNF, the set \(\var{v}\) for a leaf \(v\) is a singleton. This
   is not necessarily true in a smooth BDMC, since if \(v\) is associated with
   formula \(\varphi_i(\vek{x}_i, \vek{y}_i)\), then \(\var{v}\) contains all the
   variables \(\vek{x}_i\).
In both DNNF and BDMC, the set of leaves $v$ of a minimal satisfying subtree
is such that the corresponding sets $\var{v}$ form a partition of the set of
variables $\var{x}$. In DNNF, the partition is a cover by singletons while
for BDMC this is a general partition. This implies a difference in the
structure of the resulting subtree, however, the construction proceeds
in the same way.

The obtained minimal satisfying subtree \(T\) uses only nodes in \(V'\),
contains \(v_0\) and all nodes \(v\) contained in a positive literal in \(\alpha\).
The assignment $\assign{a}$ of the variables $V$ such that
$\assign{a}(v)=1$ if and only if $v \in T$ is a model of \(\psi_p\)
consistent with \(\alpha\) and \(v_0\).
\end{proof}

\begin{lemma} \mlabel{lem:base-urc-pc}
   The formula \(\psi_p\) is PC\@.
\end{lemma}

\begin{proof}
Let \(\alpha\subseteq
      \lit{V}\) be a partial assignment for which
   \(\psi_p\land\alpha\not\vdash_1\bot\). Let \(l\in\lit{V}\) be a
   literal such that \(\psi_p\land\alpha\not\vdash_1\neg l\) and let us show
   that \(\psi_p\land\alpha\land l\) is satisfiable. If \(l\) is
   positive, then satisfiability of \(\psi_p\land\alpha\land l\)
   directly follows by Lemma~\ref{lem:bdmc-notbot-sat}.
   Assume \(l=\neg v\)
   for some \(v\in V\), thus \(\psi_p\land\alpha\not\vdash_1 v\). Let
   us consider any separator \(S\in\mathcal{S}\) containing \(v\).
   Due to clauses of group~\ref{N6} we have that there is a node
   \(v_0\in S\) for which \(\psi_p\land\alpha\not\vdash_1\neg v_0\).
   It follows
   by Lemma~\ref{lem:bdmc-notbot-sat} that \(\psi_p\land\alpha\land
      v_0\) has a model. By clauses of group~\ref{N6} we have that
   \(v\) gets value \(0\) in any model satisfying \(v_0\), thus
   \(\psi_p\land\alpha\land\neg v=\psi_p\land\alpha\land l\) is
   satisfiable as well.
\end{proof}

Let \(\encPC'(\vek{z}, \vek{v})\) be the subformula of \(\encPC(\vek{x},
\vek{z}, \vek{v})\) formed by the conjunction of $\psi_p$ and
the clauses of group~\ref{E1}.

\begin{lemma}
   \mlabel{lem:combine-urc-pc}
The formula \(\encPC'(\vek{z}, \vek{v})\) is PC\@.
\end{lemma}
\begin{proof}
By Lemma~\ref{lem:base-urc-pc}, \(\psi_p\) is PC\@.
Let us set \(\psi_0=\psi_p\). It was shown by~\citet{BM12} that a
   formula which is composed as a conjunction of two PC formulas
   sharing a single variable is PC\@. By inductive use of this
   argument together with Lemma~\ref{lem:xdrenc-urc-pc} we obtain
   that for each \(i=1, \dots, \ell\) the formula
   \(\psi_i=\psi_{i-1}\land\xdrenc{\varphi_i}{\vek{z}_i}\) is PC\@. The
   proposition follows since \(\encPC'(\vek{z}, \vek{v})=\psi_\ell\).
\end{proof}

The following lemma allows to extend models of \(\encPC'\) to
models of \(\encPC\).

\begin{lemma}
   \mlabel{lem:urc:models}
   Let \(\assign{b}\cup\assign{c}\) be a model of \(\encPC'(\vek{z},
      \vek{v})\) where \(\assign{b}\) assigns values to \(\vek{z}\)
   and \(\assign{c}\) assigns values to \(\vek{v}\). Then
the equivalences
   \begin{equation}\label{eq:urc:models}
      l \Leftrightarrow \bigwedge_{i \in \range{l}}
      \metaVar{i}{l}
   \end{equation}
for every \(l\in\lit{\vek{x}}\) are satisfiable by a uniquely determined
assignment $\assign{a}$ of values to the variables $\vek{x}$.
Moreover, \(\assign{a}\cup\assign{b}\cup\assign{c}\) forms a model of
\(\encPC(\vek{x}, \vek{z}, \vek{v})\) and the assignment $\assign{a}$
satisfying~\eqref{eq:urc:models} is the only assignment with this property.
\end{lemma}
\begin{proof}
The identities~\eqref{eq:urc:models} are equivalent to
the clauses of groups~\ref{E2} and~\ref{E3} and if they are
satisfiable for a given assignment of the variables $\vek{z}$,
they uniquely determine an assignment of the values to the
variables $\vek{x}$. Hence, it is sufficient to
prove that~\eqref{eq:urc:models} is satisfiable,
if $\encPC'(\vek{b}, \vek{c})$ is satisfied.

Consider any literal \(l\in\lit{\vek{x}}\).
Model \(\assign{b}\cup\assign{c}\) specifies a minimal satisfying subtree
\(T\) of \(D\). Since \(D\) is decomposable and smooth, \(T\) contains
a leaf \(v=\metaVar{j}{\top}\) for exactly one \(j\in\range{l}\).
It follows that $\assign{b}(\metaVar{j}{\top})=1$ while 
$\assign{b}(\metaVar{i}{\top})=0$ for every
\(i\in\range{l}\setminus\{j\}\).
Moreover, for every \(i\in\range{l}\setminus\{j\}\), we have
   \(\assign{b}(\metaVar{i}{l})=\assign{b}(\metaVar{i}{\neg l})=1\)
   due to clauses \(\metaVar{i}{\top}\vee\metaVar{i}{l}\) and
   \(\metaVar{i}{\top}\vee\metaVar{i}{\neg l}\) in
\(\xdrenc{\varphi_i}{\vek{z}_i}\).
On the other hand, since $\assign{b}$ satisfies
$\xdrenc{\varphi_j}{\vek{z}_j} \wedge  \metaVar{j}{\top}$,
we have $\assign{b}(\metaVar{j}{l}) = \neg \assign{b}(\metaVar{j}{\neg l})$.

It follows that for every variable $x \in \vek{x}$, the right hand sides
of~\eqref{eq:urc:models} for $l=x$ and $l = \neg x$ have complementary
values. Hence, it is possible to assign a value to $x$, such
that~\eqref{eq:urc:models} is satisfied for $l \in \{x, \neg x\}$.
For a fixed assignment of the variables $\vek{z}$,
the equivalences~\eqref{eq:urc:models} for literals $l$ on different variables
from $\vek{x}$ are independent. This completes the proof.
\end{proof}

The main result of this section is the following.

\begin{proposition} \mlabel{prop:main-PC}
   Formula \(\encPC(\vek{x}, \vek{z}, \vek{v})\) is propagation complete.
\end{proposition}
\begin{proof}
   Let \(\alpha\subseteq\lit{\vek{x}\cup\vek{z}\cup\vek{v}}\) be a
   partial assignment closed under unit propagation in
   \(\encPC(\vek{x}, \vek{z}, \vek{v})\). In particular,
   \(\encPC(\vek{x}, \vek{z}, \vek{v})\land\alpha\not\vdash_1\bot\).
   Let \(l\in\lit{\vek{x}\cup\vek{z}\cup\vek{v}}\) such that
   \(\neg l\not\in\alpha\). We shall show that \(\encPC(\vek{x},
      \vek{z}, \vek{v})\land\alpha\land l\) is satisfiable.

   Let us first assume that \(l\in\lit{\vek{z}\cup\vek{v}}\). Consider
   the restriction \(\alpha'=\alpha\cap\lit{\vek{z}\cup\vek{v}}\).
We have \(\encPC'(\vek{z},
      \vek{v})\land\alpha'\not\vdash_1\bot\) and \(\encPC'(\vek{z},
      \vek{v})\land\alpha'\not\vdash_1 \neg l\). It follows by
   Lemma~\ref{lem:combine-urc-pc} that \(\encPC'(\vek{z},
      \vek{v})\land\alpha'\land l\) has a model
   \(\assign{b}\cup\assign{c}\) where \(\assign{b}\) assigns values to
   \(\vek{z}\) and \(\assign{c}\) assigns values to \(\vek{v}\). We have by Lemma~\ref{lem:urc:models} that there is
   an assignment \(\assign{a}\) of values to \(\vek{x}\) such that
   \(\assign{a}\cup\assign{b}\cup\assign{c}\) forms a model of
   \(\encPC(\vek{x}, \vek{z}, \vek{v})\land\alpha'\land l\). Due to
   clauses of group~\ref{E2} we have for every literal
   \(e\in\alpha\cap\lit{\vek{x}}\) and every \(i\in\range{e}\) that
   \(\metaVar{i}{e}\in\alpha'\) and thus
   \(\assign{b}(\metaVar{i}{e})=1\). By~\eqref{eq:urc:models} we get
   that \(\assign{a}\) is consistent with \(\alpha\cap\lit{\vek{x}}\)
   and thus \(\assign{a}\cup\assign{b}\cup\assign{c}\) is a model of
   \(\encPC(\vek{x}, \vek{z}, \vek{v})\land\alpha\land l\).

Let us now assume that \(l\in\lit{\vek{x}}\). By assumption \(\neg
      l\not\in\alpha\) thus there is some \(i\in\range{l}\) for which
   \(\metaVar{i}{\neg l}\not\in\alpha\) (otherwise \(l\) would be
   derived using a clause of group~\ref{E3}). It follows by the
   previous case that formula \(\encPC(\vek{x}, \vek{z},
      \vek{v})\land\alpha\land\neg\metaVar{i}{\neg l}\) is
   satisfiable. Due to clause \(\neg l\to\metaVar{i}{\neg l}\) of
   group~\ref{E2} we have that any model of \(\encPC(\vek{x}, \vek{z},
      \vek{v})\land\alpha\land\neg\metaVar{i}{\neg l}\) satisfies
   \(l\) and thus \(\encPC(\vek{x}, \vek{z}, \vek{v})\land\alpha\land
      l\) is satisfiable.
\end{proof}

\section{Conclusion and Further Research} 
\label{sec:conclusion} 

We have introduced the language of $\mathcal{C}$-BDMCs which is a
common generalization of DNNFs~\citep[introduced by][]{D99} and
$\mathcal{C}$-backdoor trees~\citep[introduced by][]{SS08}.
Moreover, URC-BDMCs contain the disjunctive closure of URC encodings
$\operatorname{\mathtt{URC-C}}[\lor, \exists]$~\cite{BJM12}
as a subset.

We have shown that PC-BDMCs
are polynomially equivalent to PC encodings~\citep[introduced
by][]{BM12}. In particular, the language of disjunctions of PC encodings
is polynomially equivalent to PC encodings. By a slight modification of
our construction, we get a similar result for URC encodings which is
a generalization of the results of~\citet{BJM12}.
We have demonstrated that PC-BDMCs and PC encodings have the same
properties as DNNFs with respect to query answering and transformations
considered in the knowledge compilation map~\cite{DM02}.
Using the results of~\citet{BCMS14,BCMS16},
PC encodings and PC-BDMCs are strictly more succinct than
DNNFs and we have shown that they are also strictly more succinct than
generalized PC-backdoor trees we have introduced
in Section~\ref{sec:relations}.

Although PC encodings and PC-BDMCs are polynomially equivalent, a
compilation from a CNF into a PC-BDMC can be easier than to a PC
encoding. In particular, we might consider modifications of the
techniques of compiling a CNF into a DNNF~\citep[see
e.g.,][]{D04,LM17,MMBH12} such that the process of splitting the
formula into pieces stops at encodings from a class $\mathcal{C}$,
so earlier than at the literals. Note that some
of the compilers into a DNNF~\citep[e.g., D4 introduced by][]{LM17}
actually compile into a Decision DNNF which is a DNNF where only
decision nodes and decomposable \(\land\)-gates are allowed. It is
natural to consider $\mathcal{C}$-BDMCs restricted in the same way.
Such Decision $\mathcal{C}$-BDMCs form a language more restricted than
general $\mathcal{C}$-BDMCs, but still more general than
generalized $\mathcal{C}$-backdoor trees.
Further research is needed to determine whether one can construct
a reasonably efficient compiler of a CNF formula into
a \(\mathcal{C}\)-BDMC which uses for example prime 2-CNFs (which are
PC) or renamable Horn formulas (which are URC) as the base class
$\mathcal{C}$.


\section*{Acknowledgements}

Petr Ku{\v c}era acknowledges the support by
Czech-French Mobility program (Czech Ministry of Education, grant
No.~7AMB17FR027). Both authors acknowledge the
support by Grant Agency of the Czech Republic (grant No.~GA19--19463S).
Both authors would like to thank Pierre Marquis,
Jean-Marie Lagniez, Gilles Audemard, and Stefan Mengel (from CRIL,
U. Artois, Lens, France) for discussion on the compilation of CNF formulas
into more general structures than DNNF\@.

\bibliography{ms}
\bibliographystyle{plainnat}

\end{document}